\documentclass[letterpaper, 10 pt, conference]{ieeeconf}  
\usepackage{setspace}
\usepackage{multirow}
\usepackage{booktabs}
\usepackage{verbatim} 
\usepackage{empheq}
\usepackage[normalem]{ulem}
\usepackage{soul}
\usepackage{xcolor}
\usepackage{algorithm,algorithmic}
\usepackage{bbm}
\usepackage{amsmath,amssymb}
\usepackage{balance}
\allowdisplaybreaks[2]
\usepackage{amsfonts}
\usepackage{graphicx}
\graphicspath{{./figures/}}
\usepackage{subcaption}
\usepackage{pict2e,color}
\usepackage{cite}

\newtheorem{theorem}{Theorem}
\newtheorem{lemma}{Lemma}
\newtheorem{proposition}{Proposition}
\newtheorem{corollary}{Corollary}

\newtheorem{assumption}{Assumption}

\IEEEoverridecommandlockouts                              

\usepackage{tikz}
\usepackage{textcomp}
\usepackage{hyperref}
\usepackage{lipsum}

\usepackage{graphics} 
\usepackage{epsfig} 
\usepackage{times} 
\usepackage{amsmath} 
\usepackage{amssymb}  

\title{\LARGE \bf
Residuals-based Offline Reinforcement Learning
}

\author{Qing Zhu$^{1}$ and Xian Yu$^{2}$
\thanks{This work was supported in part by the National Science Foundation under Grant \#2331782.}
\thanks{$^{1}$Department of Integrated Systems Engineering, The Ohio State University, Columbus, OH, USA,
        {\tt\small zhu.2166@osu.edu}}%
\thanks{$^{2}$Corresponding author; Department of Integrated Systems Engineering, The Ohio State University, Columbus, OH, USA,
        {\tt\small yu.3610@osu.edu}}%
}

\begin{document}

\maketitle
\thispagestyle{empty}
\pagestyle{empty}

\begin{abstract}
Offline reinforcement learning (RL) has received increasing attention for learning policies from previously collected data without interaction with the real environment, which is particularly important in high-stakes applications. 
While a growing body of work has developed offline RL algorithms, these methods often rely on restrictive assumptions about data coverage and suffer from distribution shift.
In this paper, we propose a residuals-based offline RL framework for general state and action spaces. Specifically, we define a residuals-based Bellman optimality operator that explicitly incorporates estimation error in learning transition dynamics into policy optimization by leveraging empirical residuals. We show that this Bellman operator is a contraction mapping and identify conditions under which its fixed point is asymptotically optimal and possesses finite-sample guarantees. We further develop a residuals-based offline deep Q-learning (DQN) algorithm. Using a stochastic CartPole environment, we demonstrate the effectiveness of our residuals-based offline DQN algorithm.

\end{abstract}

\section{INTRODUCTION}
Sequential decision-making under uncertainty is fundamental to many real-world applications. In these problems, agents need to make decisions through interacting with the environment to minimize some cumulative costs (or maximize some cumulative rewards). Markov decision processes (MDPs) provide a natural framework for modeling such problems \cite{puterman2014markov}, and reinforcement learning (RL) has become a popular tool for solving them \cite{sutton2018reinforcement}. However, online RL is rarely deployed in high-stakes applications, such as transportation, energy, and healthcare systems, due to its reliance on interaction with the real environment. Specifically, online RL relies on trial-and-error learning, during which the decisions may be highly suboptimal, potentially resulting in catastrophic outcomes. In this paper, we focus on the \textit{offline} RL setting instead, where learning is based on a static logged dataset collected under a behavior policy. We refer interested readers to an extensive review on offline RL \cite{levine2020offline}.

The essential limitations in offline RL are threefold. First, since the learning algorithm entirely relies on the static offline dataset, this dataset must have sufficient coverage over the state-action space. A common assumption requires that the state-action distribution induced by the target policy is absolutely continuous with respect to the behavior distribution, with a uniformly bounded density ratio (often referred to as a \textit{concentrability coefficient}) \cite{munos2005error,munos2008finite,chen2019information,jiang2025offline}. This assumption is often prohibitive in practice, especially for problems with continuous state and/or action spaces, as it requires the data distribution to uniformly cover all reachable states.
Second, offline RL is often concerned about counterfactual queries, i.e., asking what might have happened if the agent had implemented a different action than the one in the dataset. This is impossible in the high-stakes settings, as it is impractical to experiment with different decisions solely to observe the resulting state. Third, a fundamental challenge in offline RL is the \textit{distribution shift} between the behavior policy governing the offline dataset and the policy being trained, leading to unreliable policy evaluation. 
Distribution shift can be addressed by constraining the distance between the learned policy and the behavior policy \cite{kakade2002approximately,schulman2015trust,fujimoto2019off,kumar2019stabilizing,wu2019behavior}. However, since the past decisions in the offline dataset may be suboptimal, enforcing such constraints may prevent the learned policy from achieving optimal performance.

In this paper, we propose a residuals-based offline RL framework for general state and action spaces. In contextual stochastic optimization, \cite{kannan2025data} proposed an empirical residuals-based sample average approximation (ER-SAA) for single-stage and two-stage stochastic programs that directly incorporates the estimation error in predicting the uncertainty into the downstream optimization problem.
Leveraging this idea, in this work, we first construct an estimated transition model from the static offline data using supervised learning. Based on the observed transitions, we compute the empirical residuals, which capture the discrepancy between the learned model and the true system dynamics. We generate trajectories by sampling residuals and adding them to the point prediction of the next state, and then train our policies using these generated trajectories. This framework addresses the current limitations in offline RL in the following ways. First, different than the vast literature, we do not require the dataset to cover all the state-action pairs, as the sampling procedure will automatically generate states that are unseen in the dataset due to the empirical residuals. Second, this residuals-based framework emulates real-world dynamics and provides approximate answers to counterfactual queries. Third, this framework addresses distribution shift by collecting data following the learned policy and using it to evaluate the value function, enabling \textit{on-policy} training.

The rest of this paper is organized as follows. We first introduce our infinite-horizon discounted MDP setting and review ER-SAA in Section~\ref{sec:pre}. In Section~\ref{sec:offlineRL}, we develop the residuals-based offline RL framework and establish its consistency and finite sample guarantees. For practical use, we also design a residuals-based offline deep Q-learning (DQN) algorithm. Finally, we present our numerical results on a stochastic CartPole environment in Section~\ref{sec:results}.

\section{PRELIMINARIES}\label{sec:pre}
\subsection{Infinite-horizon Discounted MDPs}

We consider an infinite-horizon $\gamma$-discounted MDP represented by the tuple $(\mathcal S,\mathcal A, P, c, \gamma)$, where $\mathcal S$ is the convex state space, $\mathcal A$ is the compact (possibly continuous) action space, $P(\cdot|s,a)$ is the conditional density of the transition distribution, $c(s,a)$ is the deterministic immediate cost function\footnote{Our framework can be extended to incorporate stochastic immediate cost with relative ease.}, and $\gamma\in[0,1)$ is the discount factor. Note that we require a convex state space to ensure the Lipschitz continuity of the orthogonal projection $\Pi$, i.e., $\|\Pi_{\mathcal S} a - \Pi_{\mathcal S} b\| \le \|a-b\|,\ \forall a,b$.

A stationary deterministic Markov policy is a mapping $\pi:\mathcal S\to\mathcal A$, which specifies the action $a=\pi(s)$ in state $s$. Let $M$ denote the set of all stationary deterministic Markov policies. For any policy $\pi \in M$, the corresponding state value function and state-action value function are defined as
$V^\pi(s):=
\mathbb E^\pi\left[
\sum_{t=0}^{\infty}\gamma^t c(s_t,a_t)|s_0=s\right],\ a_t=\pi(s_t),\ s_{t+1}\sim P(\cdot|s_t,a_t)$
and $Q^\pi(s,a) :=
\mathbb E^\pi\!\left[
\sum_{t=0}^{\infty}\gamma^t c(s_t,a_t)
\,\middle|\, s_0=s,\ a_0=a
\right],\ a_t=\pi(s_t),\ s_{t+1}\sim P(\cdot|s_t,a_t)$, respectively.
It is well known that there exists an optimal stationary deterministic Markov policy in risk-neutral RL \cite{puterman2014markov}. Without loss of optimality, our goal is to find a policy $\pi \in M$ that minimizes the expected total discounted cost: $\min_{\pi\in M} V^{\pi}(s)$.

Denote the optimal policy as $\pi^\star$ and the optimal value function and state-action value function by $V^\star(s)$ and $Q^\star(s,a)$, where $V^\star(s)=\min_{a\in \mathcal A}Q^\star(s,a)=Q^\star(s,\pi^\star(s))$. According to \cite{puterman2014markov}, they satisfy the following Bellman equations 
\begin{align*}
&V^\star(s)
=
\min_{a\in\mathcal A}
\left\{
c(s,a)
+
\gamma\,
\mathbb E\!\left[
V^\star(s')
\,\middle|\, s'\sim P(\cdot\mid s,a)
\right]
\right\},\\
&Q^\star(s,a)=
c(s,a)+\gamma\,
\mathbb E\!\left[
\min_{a'\in\mathcal A} Q^\star(s',a')
\,\middle|\, s'\sim P(\cdot\mid s,a)
\right].
\end{align*}

Let $f^\star$ denote the ground-truth (unknown) transition function such that $s' = f^\star(s,a) + \epsilon$, where $\epsilon\sim P_{\epsilon}$ is the zero-mean random noise. Then, the conditional density $P(\cdot|s,a)$ can be equivalently converted to the density $P_{\epsilon}$. Specifically, let $ T^\star$ denote the Bellman optimality operator and $\mathcal B(\mathcal S,\mathcal A)$ denote the set of all bounded measurable real-valued functions on $\mathcal S \times \mathcal A$ with supremum norm $\|Q\|_\infty := \sup_{(s,a) \in \mathcal S \times \mathcal A} |Q(s,a)|$. Then, for any function $Q \in \mathcal B(\mathcal S,\mathcal A)$, we have
\begin{align}\label{eq: true bellman operator}
&(T^\star Q)(s,a)
:= c(s,a)
+ \gamma \int_{\mathcal{S}}
\min_{a'\in\mathcal{A}} Q(s',a')\,
P(\mathrm{d}s'\mid s,a)\nonumber\\
&=
c(s,a)
+\gamma \int \min_{a' \in \mathcal A} Q\bigl(f^\star(s,a)+\epsilon,\; a'\bigr)\,P_{\epsilon}(d\epsilon).
\end{align}
From \cite{puterman2014markov}, the Bellman optimality operator defined in\ \eqref{eq: true bellman operator} is a contraction mapping, where $Q^\star$ is its unique fixed point. 

\noindent\textbf{Offline RL.}
In this paper, we study an offline RL setting in which learning is based on a fixed batch dataset $D_N=\{(s^{(i)},a^{(i)},c^{(i)},s'^{(i)})\}_{i=1}^N$ instead of on direct interaction with the environment. The dataset is generated according to a behavior distribution $d^b$ and the transition kernel $P$, i.e., $(s^{(i)},a^{(i)})\sim d^b,\ c^{(i)}=c(s^{(i)},a^{(i)}),\ s'^{(i)}\sim P(\cdot\mid s^{(i)},a^{(i)}),\ i=1,\dots,N.$ Given a test state distribution $\rho\in\Delta(\mathcal S)$, the aim of offline RL is to learn a policy $\pi$ such that for a prescribed accuracy level $\epsilon>0$, we have $V^\star(\rho)-V^\pi(\rho)\le \epsilon$ with high probability. 

\subsection{ER-SAA}
Next, we review the ER-SAA proposed in \cite{kannan2025data}. Contextual stochastic optimization studies decision-making under uncertainty by leveraging contextual information that affects the distribution of random outcomes. Specifically, given a new observation of the random contextual information $X = x$, we aim to solve the following stochastic optimization problem $\min_{z \in \mathcal Z} \mathbb E[c(z,Y)|X=x]$, where $z \in \mathcal Z$ denotes the decision variables, $Y \in \mathcal Y$ denotes the uncertainty, $c$ is a cost function, and the expectation is computed with respect to the distribution of $Y$ conditioned on $X=x$.

Assume that the true relationship between the random vector $Y$ and the random covariates $X$ is given by $Y = f^\star(X) + \varepsilon$, where $f^\star(x):=\mathbb E[Y|X=x]$ is the true regression function, and $\varepsilon$ is a zero-mean random error independent of $X$. 
In practice, $f^\star$ is usually unknown and can be estimated from some dataset $D_N=\{(x^{i},y^{i})\}_{i=1}^N$. In \cite{kannan2025data}, the authors first estimate $f^\star$ with $\hat f_N$ based on dataset $D_N$ and define the empirical residuals as $\hat{\epsilon}_N^k := y^k - \hat f_N(x^k), \ k=1,\dots,N.$ Using $\hat f_N$ and the empirical residuals, one can construct the ER-SAA as follows:
\begin{align}\label{eq:ersaa}
\min_{z\in\mathcal Z}\frac{1}{N}\sum_{k=1}^N c\!\left(z,\Pi_{\mathcal Y}\big(\hat f_N(x)+\hat\epsilon_N^k\big)\right),
\end{align}
where $\Pi_{\mathcal Y}$ ensures that the predicted scenarios lie within $\mathcal Y$.

In \cite{kannan2025data}, the authors derived conditions under which the solutions of the ER-SAA \eqref{eq:ersaa} are asymptotically optimal and possess finite sample guarantees. However, this framework has only been discussed in the single-stage or two-stage setting when the uncertainty is solely affected by the contextual information. Next, we apply the empirical residuals-based framework to offline RL, where the state can be affected by both the previous state and action, and we learn such transition dynamics using an offline dataset.

\section{RESIDUALS-BASED OFFLINE RL}\label{sec:offlineRL}
\allowdisplaybreaks
In this section, we introduce the residuals-based offline RL framework (see Fig.\ \ref{fig:flowchart}), which consists of two main steps: (i) \textbf{Regression}, where we estimate the transition function based on the offline dataset using supervised learning and construct the empirical residuals and (ii) \textbf{Residuals-based offline RL}, where we generate next states by sampling empirical residuals and adding them to the point prediction and train policies based on the generated trajectories. 

Next, we first define the residuals-based and full-information Bellman optimality operators and study the property of their fixed points in Section~\ref{sec:operator}. Then we prove the consistency and finite sample guarantees of the residuals-based offline RL framework in Section~\ref{sec:consistency}.
\begin{figure}[ht!]
    \centering
\includegraphics[width=\linewidth]{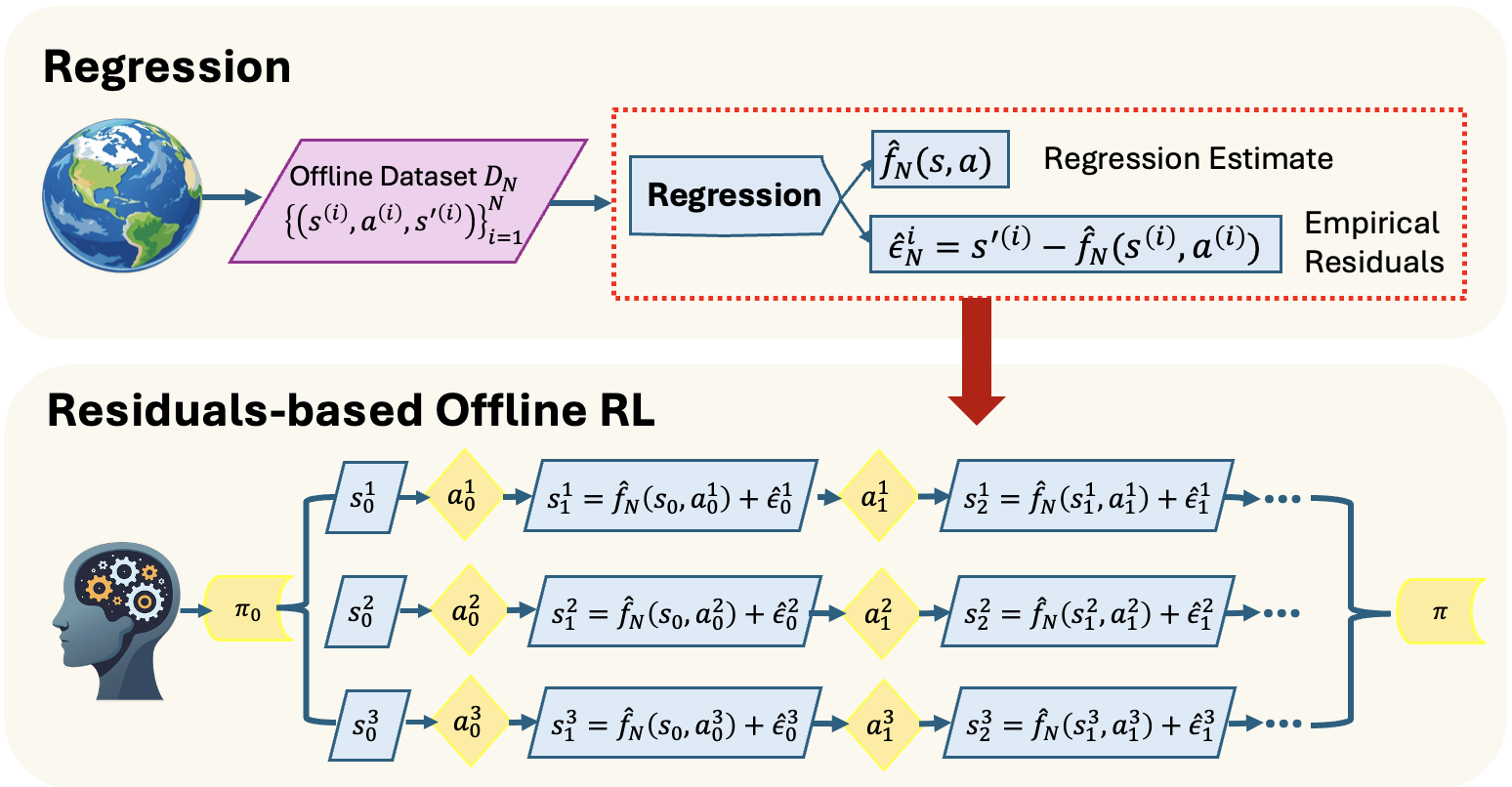}
    \caption{Flowchart of the residuals-based offline RL.}
    \label{fig:flowchart}
\end{figure}

\subsection{Residuals-based Bellman Optimality Operator}\label{sec:operator}

Because the true transition function $f^\star$ is unknown, we estimate it from the offline dataset $D_N =\{(s^{(i)},a^{(i)},c^{(i)},s'^{(i)})\}_{i=1}^N$ using supervised learning. Let $\mathcal{F}$ denote a function class that contains $f^*$\footnote{See Remark 1 in \cite{kannan2025data} when $f^*\not\in\mathcal F$. In this misspecified setting, we can replace $f^*$ with the best approximator in the class in our analysis.}. Given dataset $D_N$ and some loss function $\ell: \mathcal{S} \times \mathcal{S} \to \mathbb{R}_+$, we obtain an estimator $\hat f_N$ by solving an empirical risk minimization problem $\hat f_N \in \arg\min_{f \in \mathcal{F}} \frac{1}{N} \sum_{i=1}^N \ell\bigl(f(s^{(i)}, a^{(i)}),\, s'^{(i)}\bigr)$. We then construct the empirical residuals $\hat\epsilon_N^{(i)}:=s'^{(i)}-\hat f_N(s^{(i)},a^{(i)}), \ i=1,\ldots,N$, which measure the discrepancy between the observation and the predicted next state. Finally, given any state-action pair $(s,a)$, we add these $N$ empirical residuals to the point prediction of the next state and project them onto the state space $\mathcal S$, i.e., $\Pi_{\mathcal S}(\hat f_N(s,a)+\hat\epsilon_N^{(i)}),\ \forall i=1,\ldots,N$. We then approximate the true transition kernel with the empirical transition kernel $\hat P_N(\cdot|s,a)
:=\frac{1}{N}\sum_{i=1}^N \delta_{\Pi_{\mathcal S}(\hat f_N(s,a)+\hat\epsilon_N^{(i)})}(\cdot)$ supported on the $N$ scenarios with equal probability $1/N$, where $\delta(\cdot)$ denotes the Dirac measure. Based on this empirical transition kernel, we define the corresponding \textbf{residuals-based Bellman optimality operator} as follows
\begin{align}\label{eq:empirical bellman operator}
&(\hat{T}_N Q)(s,a)
:= c(s,a)+\gamma \int_{\mathcal S}\min_{a'\in\mathcal A}Q(s',a')\,\hat P_N(\mathrm ds'\mid s,a) \nonumber\\
&= c(s,a)+ \frac{\gamma}{N}\sum_{i=1}^N
\min_{a'\in\mathcal A}Q\!\Bigl(\Pi_{\mathcal S}(\hat f_N(s,a)+\hat\epsilon_N^{(i)}),a'\Bigr).
\end{align}

To facilitate analysis, we also introduce the following full-information Bellman optimality operator. Under the true regression $f^\star$, we use the residuals $\epsilon^{(i)} = s'^{(i)} - f^\star(s^{(i)},a^{(i)}),\ i=1,\ldots, N$ to approximate the true error $\epsilon$. Correspondingly, we construct the full-information transition kernel as $P_N^\star(\cdot|s,a)
:=\frac{1}{N}\sum_{i=1}^{N}\delta_{f^\star(s,a)+\epsilon^{(i)}}(\cdot)$. Note that because $f^\star(s,a)+\epsilon^{(i)}$ falls within the state space $\mathcal S$, we do not need a projection here. As a result,\ \eqref{eq: true bellman operator} can be approximated with the \textbf{full-information Bellman optimality operator} as
\allowdisplaybreaks
\begin{align}\label{eq:full info bellman operator}
&(T_N^\star Q)(s,a)
:= c(s,a) + \gamma \int_{\mathcal S} \min_{a'\in\mathcal A} Q(s',a')\,
P_N^\star(\mathrm ds'\mid s,a) \nonumber \\
&=  c(s,a) + \frac{\gamma}{N} \sum_{i=1}^N \min_{a' \in \mathcal A} Q(f^\star(s,a)+\epsilon^{(i)}
,a').
\end{align}

In practice, this full-information Bellman optimality operator cannot be implemented because both the true transition model $f^\star$ and residuals $\epsilon^{(i)}$ remain unknown. Instead, we will implement the residuals-based Bellman optimality operator based on the estimated transition model $\hat{f}_N$ and empirical residuals $\hat{\epsilon}_N$. However, the full-information Bellman optimality operator will be useful in the analysis to bound the gap between $T^\star$ and $\hat{T}_N$ as an intermediate step. In the remainder of this section, we study the Lipschitz continuity of the fixed points induced by the full-information Bellman optimality operator $T_N^\star$ and the true Bellman optimality operator $T^\star$. We begin by showing that both $\hat{T}_N$ and $T_N^\star$ are contraction mappings in the next theorem, which guarantees the existence and uniqueness of their fixed points. 

\begin{theorem}\label{thm:contraction}
For $0\le \gamma<1$, the residuals-based Bellman optimality operator $\hat{T}_N$ defined in\ \eqref{eq:empirical bellman operator} and the full-information Bellman operator $T^\star_N$ defined in\ \eqref{eq:full info bellman operator} are $\gamma$-contraction mappings, i.e., $\forall\, Q^{1},\, Q^{2} \in \mathcal{B}(\mathcal{S}\times\mathcal{A})$, we have
\begin{align*}
&\|\hat{T}_N Q^{1}-\hat{T}_N Q^{2}\|_\infty
\le \gamma \|Q^{1}-Q^{2}\|_\infty,\\
&\|T_N^\star Q^{1}-T_N^\star Q^{2}\|_\infty
\le \gamma \|Q^{1}-Q^{2}\|_\infty 
\end{align*}
\end{theorem}

\begin{proof}
Given a state-action pair $(s,a)\in\mathcal{S}\times\mathcal{A}$, based on\ \eqref{eq:empirical bellman operator} and denoting
$s_i':=\Pi_{\mathcal S} (\hat f_{N}(s,a)+\hat\epsilon_N^{(i)})$ for $i=1,\ldots,N$, we have
\allowdisplaybreaks
\begin{align}\label{eq:contraction proof for hat T}
& \bigl|(\hat{T}_N Q^{1})(s,a)-(\hat{T}_N Q^{2})(s,a)\bigr| \nonumber\\
&= \left|\frac{\gamma}{N}\sum_{i=1}^{N}
\left[
\min_{a'\in\mathcal A} Q^{1}\!\bigl(s'_i,a'\bigr)
- \min_{a'\in\mathcal A} Q^{2}\!\bigl(s'_i,a'\bigr)
\right]\right| \nonumber \\
&\le \frac{\gamma}{N}\sum_{i=1}^{N} \sup_{a'\in\mathcal A}
\left|Q^{1}(s'_i,a')-Q^{2}(s'_i,a')\right| \nonumber \\
&\le \gamma \sup_{(s',a')\in\mathcal{S}\times\mathcal{A}}
\left|Q^{1}(s',a')-Q^{2}(s',a')\right| = \gamma \|Q^{1}-Q^{2}\|_\infty.
\end{align}
Taking the supremum over $(s,a)\in\mathcal{S}\times\mathcal{A}$ on both sides yields $\|\hat{T}_N Q^{1}-\hat{T}_N Q^{2}\|_\infty
\le \gamma \|Q^{1}-Q^{2}\|_\infty$.
The proof for the full-information Bellman optimality operator $T_N^\star$ proceeds in the same way as\ \eqref{eq:contraction proof for hat T}, with $s_i'=\Pi_{\mathcal S}(\hat f_N(s,a)+\hat\epsilon_N^{(i)})$ replaced by $s_i'=f^\star(s,a)+\epsilon^{(i)}$ based on the definition of $T_N^\star$ defined in\ \eqref{eq:full info bellman operator}. 
\end{proof}

Since both residuals-based and full-information Bellman optimality operators are contraction mappings, applying them iteratively will lead to their unique fixed points, denoted by $\hat{Q}_N$ and $Q_N^\star$, respectively. To derive the Lipschitz continuity of $Q_N^\star$, we first make the following assumptions.

\begin{assumption}\label{ass:cost}
The immediate cost $c(s,a)$ is Lipschitz continuous on $s \in \mathcal S$ uniformly over $a \in \mathcal A$, i.e., $\sup_{a\in\mathcal A}|c(s_1,a)-c(s_2,a)| \le L_c\|s_1-s_2\|,\ \forall s_1,s_2 \in \mathcal S$
with some Lipschitz constant $L_c<+\infty$.
\end{assumption}

\begin{assumption}\label{ass:f star lip}
The true transition function $f^\star$ is Lipschitz continuous on $s \in \mathcal S$ uniformly over $a \in \mathcal A$, i.e., $\sup_{a \in \mathcal A} \|f^\star(s_1,a)-f^\star(s_2,a)\| \le L_f\|s_1-s_2\|,\ \forall s_1,s_2 \in \mathcal S$ with some Lipschitz constant $L_f < +\infty$.
\end{assumption}

Assumption\ \ref{ass:f star lip} can be satisfied by many commonly used transition models under a compact action space $\mathcal A$.

Next, we show that the fixed point of the full-information Bellman optimality operator ($Q_N^\star$) is Lipschitz continuous.
\begin{proposition}[Lipschitz continuity of $Q_N^\star$]\label{prop:QNs_lip}
Suppose Assumptions\ \ref{ass:cost} and\ \ref{ass:f star lip} hold. Consider the full-information Bellman optimality operator defined in\ \eqref{eq:full info bellman operator}. If $\gamma L_f<1$, then the fixed point $Q_N^\star(\cdot,a)$ is Lipschitz continuous in $s$ uniformly over $a\in\mathcal A$, i.e.,
\[
\sup_{a\in\mathcal A}|Q_N^\star(s_1,a)-Q_N^\star(s_2,a)|
\le L_{Q^\star_N}\,\|s_1-s_2\|,\ \forall s_1,s_2\in\mathcal S,
\]
where $L_{Q^\star_N}:=\frac{L_c}{1-\gamma L_f}$. Moreover, define the value function $V_N^\star(s):=\min_{a\in\mathcal A}Q_N^\star(s,a)$. Then $V_N^\star(s)$ is also Lipschitz continuous with the same Lipschitz constant $L_{Q^\star_N}$.
\end{proposition}

\begin{proof}
Fix any initialization $Q_N^{(1)} \in \mathcal{B}(\mathcal{S}\times\mathcal{A})$ such that $Q_N^{(1)}(\cdot,a)$ is Lipschitz continuous in $s$ uniformly over $a \in \mathcal A$, i.e., $\sup_{a\in\mathcal A}|Q_N^{(1)}(s_1,a)-Q_N^{(1)}(s_2,a)|
\le L_1\|s_1-s_2\|,\ \forall s_1,s_2\in\mathcal S$
for some constant $L_1<+\infty$. Define the full-information value-iteration sequence by $Q_N^{(k+1)}:=T_N^\star Q_N^{(k)},\ \forall k\ge 1$.
We will prove the Lipschitz continuity of $Q_N^{(k)}$ for all $k\ge 1$ using mathematical induction. First of all, $Q_N^{(1)}$ is Lipschitz continuous with constant $L_1$ according to our initialization. Now suppose that $Q_N^{(k)}(\cdot,a)$ is Lipschitz continuous in $s$ uniformly over $a$ with a Lipschitz constant $L_k$, i.e., $\sup_{a\in\mathcal A}|Q_N^{(k)}(s_1,a)-Q_N^{(k)}(s_2,a)|
\le L_k\|s_1-s_2\|,\ \forall s_1,s_2\in\mathcal S$.

Define $V_N^{(k)}(s):=\min_{a'\in\mathcal A}Q_N^{(k)}(s,a')$. Then
$|V_N^{(k)}(s_1)-V_N^{(k)}(s_2)|
\le \sup_{a'\in\mathcal A}|Q_N^{(k)}(s_1,a')-Q_N^{(k)}(s_2,a')|
\le L_k\|s_1-s_2\|$.
As a result, $V_N^{(k)}$ is $L_k$-Lipschitz continuous. For any $a\in\mathcal A$ and $s_1,s_2\in\mathcal S$,
\allowdisplaybreaks
\small
\begin{align*}
&|Q_N^{(k+1)}(s_1,a)-Q_N^{(k+1)}(s_2,a)| \\
=&\Big|(T_N^\star Q_N^{(k)})(s_1,a)-(T_N^\star Q_N^{(k)})(s_2,a)\Big|\\
\le &|c(s_1,a)-c(s_2,a)|
+\frac{\gamma}{N}\sum_{i=1}^N
\Big|V_N^{(k)}\big(f^\star(s_1,a)+\epsilon^{(i)}\big) \\
& -V_N^{(k)}\big(f^\star(s_2,a)+\epsilon^{(i)}\big)\Big|\\
\overset{(a)}\le &L_c\|s_1-s_2\|
+\frac{\gamma}{N}\sum_{i=1}^N L_k\,\big\|f^\star(s_1,a)-f^\star(s_2,a)\big\|\\
\overset{(b)}\le &\big(L_c+\gamma L_f L_k\big)\|s_1-s_2\|,
\end{align*}
\normalsize
where $(a)$ holds due to Assumption\ \ref{ass:cost} and the Lipschitz continuity of $V_N^{(k)}$, and $(b)$ holds due to Assumption\ \ref{ass:f star lip}. Taking the supremum over $a\in\mathcal A$ shows that $Q_N^{(k+1)}$ is $(L_c+\gamma L_f L_k)$-Lipschitz continuous, which yields the recursion
$L_{k+1} = L_c+\gamma L_f L_k,\ \forall k\ge 1$,
with the initial value $L_1<+\infty$.
Denote $\alpha:=\gamma L_f\in[0,1)$. This recursion yields
$L_k 
= \alpha^{k-1}L_1 + \frac{1-\alpha^{k-1}}{1-\alpha}\,L_c,\ \forall k\ge 2$.
Therefore, $\lim_{k\to\infty}L_k=\frac{L_c}{1-\alpha}=\frac{L_c}{1-\gamma L_f}=L_{Q_N^\star}$.
Since $T_N^\star$ is a $\gamma$-contraction in the sup norm, $Q_N^{(k)}$ converges to the unique fixed point $Q_N^\star$. Therefore, the limit $Q_N^\star$
is also Lipschitz continuous with the constant $L_{Q_N^\star}$, i.e.,
$\sup_{a\in\mathcal A}|Q_N^\star(s_1,a)-Q_N^\star(s_2,a)|
\le \frac{L_c}{1-\gamma L_f}\,\|s_1-s_2\|,
\ \forall s_1,s_2\in\mathcal S$.

Additionally, $|V_N^\star(s_1)-V_N^\star(s_2)|
\le \sup_{a\in\mathcal A}|Q_N^\star(s_1,a)-Q_N^\star(s_2,a)|
\le \frac{L_c}{1-\gamma L_f}\,\|s_1-s_2\|,\ \forall s_1,s_2\in\mathcal S$, which shows that $V_N^\star$ is Lipschitz continuous with the same constant $L_{Q_N^\star}$.
\end{proof}

To show the Lipschitz continuity of the fixed point of the true Bellman optimality operator $Q^\star$, we need an additional assumption on the true transition kernel.
\begin{assumption}\label{ass:P1}
There exists $L_p < + \infty$ such that $\sup_{a\in\mathcal A} W_1 (P(\cdot|s_1,a), P(\cdot|s_2,a)) \le\ L_P\,\|s_1-s_2\|$ for all $s_1,\ s_2 \in \mathcal S$, where the 1-Wasserstein distance is defined by 
\begin{align*}
&W_1(P(\cdot|s_1,a), P(\cdot|s_2,a)):=\\
&\hspace{2cm}\inf_{\omega \in \Omega(P(\cdot|s_1,a), P(\cdot|s_2,a))}
\mathbb E_{(X,Y) \sim \omega} \bigl[\|X-Y\|\bigr].
\end{align*}
Here, $\Omega(\mu,\nu)$ denotes all couplings with marginals $\mu$ and $\nu$.
\end{assumption}

\begin{proposition}[Lipschitz continuity of $Q^\star$]\label{prop: Q lip continuous}
Suppose Assumptions\ \ref{ass:cost} and\ \ref{ass:P1} hold. Consider the true Bellman optimality operator defined in\ \eqref{eq: true bellman operator}.  If $\gamma L_p <1$, then the fixed point $Q^\star(\cdot,a)$ is Lipschitz continuous in $s$ uniformly over $a\in\mathcal A$, i.e.,
\[
\sup_{a \in \mathcal A} |Q^\star(s_1,a)-Q^\star(s_2,a)|\le L_{Q^\star}\|s_1-s_2\|,\ \forall s_1,s_2\in\mathcal S,
\]
where $L_{Q^\star}:=\frac{L_c}{1 - \gamma L_p}$. Moreover, define the value function $V^\star(s):=\min_{a\in\mathcal A}Q^\star(s,a)$. Then $V^\star(s)$ is also Lipschitz continuous with the same Lipschitz constant $L_{Q^\star}$.
\end{proposition}
\begin{proof}
For any given initialization $Q^{(1)} \in \mathcal{B}(\mathcal{S}\times\mathcal{A})$ such that $Q^{(1)}(\cdot,a)$ is Lipschitz continuous in $s$ uniformly over $a \in \mathcal A$ for some constant $L_1<+\infty$. Define the value-iteration sequence by $Q^{(k+1)}:=T^\star Q^{(k)},\ \forall k\ge 1$.
Following the proof of Proposition\ \ref{prop:QNs_lip}, we will use mathematical induction to prove that $Q^{(k)}(\cdot,a)$ is Lipschitz continuous in $s$ uniformly over $a\in \mathcal A$ for all $k\ge 1$. First of all, $Q^{(1)}$ is Lipschitz continuous with constant $L_1$. Now suppose $Q^{(k)}(\cdot,a)$ is Lipschitz continuous in $s$ uniformly over $a$ with a Lipschitz constant $L_k$, i.e., $\sup_{a\in\mathcal A}|Q^{(k)}(s_1,a)-Q^{(k)}(s_2,a)|
\le L_k\|s_1-s_2\|,\ \forall s_1,s_2\in\mathcal S$.

Define $V^{(k)}(s):=\min_{a'\in\mathcal A}Q^{(k)}(s,a')$. Then similarly, $V^{(k)}$ is Lipschitz continuous with constant $L_k$, i.e., $|V^{(k)}(s_1)-V^{(k)}(s_2)|\le L_k\|s_1-s_2\|$. For any $s_1,s_2\in\mathcal S$,
\small
\begin{align*}
&\sup_{a \in \mathcal A}|Q^{(k+1)}(s_1,a)-Q^{(k+1)}(s_2,a)| \\
&=\sup_{a \in \mathcal A}\Big|(T^\star Q^{(k)})(s_1,a)-(T^\star Q^{(k)})(s_2,a)\Big|\\
&\le \sup_{a \in \mathcal A}|c(s_1,a)-c(s_2,a)|\\
& +\gamma\sup_{a \in \mathcal A}
\Big|\int_{\mathcal S} V^{(k)}(s')\,P(\mathrm ds'| s_1,a)
-\int_{\mathcal S} V^{(k)}(s')\,P(\mathrm ds'| s_2,a)\Bigr|\\
&\overset{(a)}{\le} L_c\|s_1-s_2\|
+\gamma L_k L_p\,\big\|s_1 - s_2\big\|,
\end{align*}
\normalsize
where $(a)$ follows from Assumptions\ \ref{ass:cost}, \ref{ass:P1}, the Lipschitz continuity of $V^{(k)}$ with constant $L_k$, and Kantorovich-Rubinstein duality for the $1$-Wasserstein distance \cite{villani2021topics}: $W_1\bigl(P(\cdot| s_1,a),P(\cdot|s_2,a)\bigr)=\sup_{Lip(f)\le 1}\{\int f(s')P(ds'| s_1,a)-\int f(s')P(ds'|s_2,a)\}$. 
The remaining steps follow directly from the proof of Proposition\ \ref{prop:QNs_lip}. In particular, we obtain the recursion $L_{k+1}=L_c+\gamma L_p L_k,\ k\ge 1$, which yields $\lim_{k\to\infty}L_k =  \frac{L_c}{1-\gamma L_p}=L_{Q^\star}$. Since $T^\star$ is a $\gamma$-contraction in the sup norm, $Q^{(k)}$ converges to the unique fixed point $Q^\star$, which is Lipschitz continuous with constant $L_{Q^\star}$. Additionally, $|V^\star(s_1)-V^\star(s_2)|
\le \sup_{a\in\mathcal A}|Q^\star(s_1,a)-Q^\star(s_2,a)|
\le \frac{L_c}{1-\gamma L_p}\,\|s_1-s_2\|,\ \forall s_1,s_2\in\mathcal S$, which shows that $V^\star$ is Lipschitz continuous with the same constant $L_{Q^\star}$.
\end{proof}

\subsection{Consistency and Finite Sample Guarantees}\label{sec:consistency}
In this section, we establish the consistency and finite sample guarantees of the residuals-based offline RL framework. In particular, we identify conditions under which (i) the fixed point $\hat{Q}_N$ of the residuals-based Bellman optimality operator converges to the fixed point $Q^\star$ of the true Bellman operator asymptotically and (ii) the probability of the error $\hat{Q}_N-Q^\star$ exceeding a prescribed tolerance can be upper bounded. We first introduce the following assumptions on the regression estimate.
\begin{assumption}\label{ass:regression estimate consistency}
The regression estimate has the following uniform consistency property, i.e., 
\begin{align}
    & \sup_{(s,a)\in\mathcal S\times\mathcal A}\|\hat f_N(s,a)-f^\star(s,a)\|\xrightarrow{p}0, \label{eq: part 1 in regression ass}\\
    & \frac{1}{N}\sum_{i=1}^N \|\hat{f}_N(s^{(i)},a^{(i)}) - f^\star(s^{(i)},a^{(i)})\| \xrightarrow{p} 0 \label{eq: part 2 in regression ass}
\end{align}
\end{assumption}

The uniform convergence of $\hat f_N(s,a)$ over $\mathcal S\times\mathcal A$ in Assumption\ \ref{ass:regression estimate consistency} is satisfied by a wide range of regression methods under the compactness of $\mathcal S$ and $\mathcal A$, e.g., see detailed discussions under Assumption 9 in\ \cite{zhu2024residuals} for linear regression and Assumption 2 in \cite{sun5078715contextual} for hybrid linear-exponential regression. If we assume both $\mathcal S$ and $\mathcal A$ are compact, and the regression estimator is weakly consistent, then one can obtain the uniform consistency in \eqref{eq: part 1 in regression ass}.

\begin{assumption}\label{ass:LLN}
The following uniform weak law of large numbers (LLN) holds:
\small
\begin{align}
\sup_{s,a}
\left|
\frac{1}{N}\sum_{i=1}^{N} V^\star\bigl(f^\star(s,a)+\epsilon^{(i)}\bigr)
-\mathbb E_{\epsilon}\bigl[V^\star\bigl(f^\star(s,a)+\epsilon\bigr)\bigr]
\right|
\xrightarrow{p} 0.
\end{align}
\normalsize
\end{assumption}
Assumption\ \ref{ass:LLN} can be interpreted as the uniform weak LLN on the mean of the value function. Based on Theorem 7.48 of\ \cite{shapiro2021lectures}, Assumption \ref{ass:LLN} holds under the following conditions: 
\begin{enumerate}
\item $V^\star\!\bigl(f^\star (s,a)+\epsilon\bigr)$ is continuous on $(s,a) \in \mathcal S \times \mathcal A$;
\item $\{|V^\star\!\bigl(f^\star (s,a)+\epsilon\bigr)|:(s,a) \in \mathcal S \times \mathcal A\}$ is dominated by an integrable function;
\item the samples $\epsilon^{(i)}$ are i.i.d.;
\item Both $\mathcal S$ and $\mathcal{A}$ are compact.
\end{enumerate} Similar conditions also exist under non i.i.d. settings.
\begin{theorem}[Consistency of fixed point $\hat{Q}_N$]
\label{thm:consistency} Suppose Assumptions\ \ref{ass:cost}-\ref{ass:LLN} hold. The fixed point $\hat{Q}_N$ of the residuals-based Bellman optimality operator is asymptotically optimal, i.e., $\|\hat Q_N-Q^\star\|_\infty \xrightarrow{p} 0$ as $N\to\infty$.
\end{theorem}

\begin{proof}
According to the contractive property of $T^\star$ and $\hat{T}_N$, we have $Q^\star=T^\star Q^\star$, $\hat Q_N=\hat{T}_N\hat Q_N$ and
\begin{align*}
\|\hat Q_N-Q^\star\|_\infty
&= \|\hat{T}_N\hat Q_N - T^\star Q^\star\|_\infty \\
&\le \|\hat{T}_N\hat Q_N - \hat{T}_N Q^\star\|_\infty
    + \|\hat{T}_N Q^\star - T^\star Q^\star\|_\infty \\
&\le \gamma \|\hat Q_N-Q^\star\|_\infty + \|(\hat{T}_N-T^\star )Q^\star\|_\infty.
\end{align*}
Therefore,
\begin{align}\label{eq: Q and T}
\|\hat Q_N-Q^\star\|_\infty \le \frac{1}{1-\gamma}\,\|(\hat{T}_N-T^\star)Q^\star\|_\infty.
\end{align}

Next, we bound the right-hand side of\ \eqref{eq: Q and T}. Given a state-action pair $(s,a)\in\mathcal S\times\mathcal A$, according to the definitions of $T^\star$ and $\hat{T}_N$ and denoting $V^\star(s)=\min_{a\in\mathcal A}Q^\star(s,a)$, we have
\footnotesize
\begin{align*}
&\bigl|((\hat{T}_N-T^\star)Q^\star)(s,a)\bigr| \\
&=\Bigg|
\frac{\gamma}{N}\sum_{i=1}^{N}V^\star(\Pi_{\mathcal{S}}(\hat f_N(s,a)+\hat\epsilon_N^{(i)}))
- \mathbb E_\epsilon\left[V^\star\!(f^\star(s,a)+\epsilon)\right]
\Bigg|\nonumber \\
&\le \underbrace{\Bigg|
\frac{\gamma}{N}\sum_{i=1}^{N}\Bigl(
V^\star\bigl(\Pi_{\mathcal{S}}(\hat f_N(s,a)+\hat\epsilon_N^{(i)})\bigr)
-V^\star\bigl(\Pi_{\mathcal S}(f^\star(s,a)+\hat\epsilon_N^{(i)})\bigr)
\Bigr)\Bigg|}_{(i)}\\
& + \underbrace{\Bigg|
\frac{\gamma}{N}\sum_{i=1}^{N}\Bigl(
V^\star\!\bigl( \Pi_{\mathcal S}(f^\star(s,a)+\hat\epsilon_N^{(i)})\bigr)
-V^\star\!\bigl(f^\star(s,a)+\epsilon^{(i)}\bigr)
\Bigr)\Bigg|}_{(ii)}\\
& +\underbrace{ \Bigg|
\frac{\gamma}{N}\sum_{i=1}^{N}V^\star\bigl(f^\star(s,a)+\epsilon^{(i)}\bigr)
-\mathbb E_\epsilon\left[V^\star\!\bigl(f^\star(s,a)+\epsilon\bigr)\right]
\Bigg|}_{(iii)}.
\end{align*}
\normalsize
For $(i)$, we have
\allowdisplaybreaks
\small 
\begin{align}
& \left| \frac{\gamma}{N} \sum_{i=1}^N \left(
V^\star\!\bigl(\Pi_{\mathcal S}(\hat f_N(s,a)+\hat\epsilon_N^{(i)})\bigr)
-V^\star\bigl(\Pi_{\mathcal S}(f^\star(s,a)+\hat\epsilon_N^{(i)})\bigr)\right)
\right|\nonumber\\
 & \overset{(a)}{\le} \frac{\gamma}{N} \sum_{i=1}^N  L_{Q^\star}\left\|\Pi_{\mathcal S}(\hat f_N(s,a)+\hat\epsilon_N^{(i)}) - \Pi_{\mathcal S}(f^\star(s,a)+\hat\epsilon_N^{(i)})\right\| \nonumber \\
& \overset{(b)}{\le} \gamma L_{Q^\star}\|\hat f_N(s,a)-f^\star(s,a)\|,\nonumber
\end{align}
\normalsize
where $(a)$ holds by Proposition\ \ref{prop: Q lip continuous} and $(b)$ holds by the Lipschitz continuity of the orthogonal projection. Taking the supremum over $(s,a)$ on both sides and driving $N$ to $+\infty$, $(i)$ converges to $0$ in probability due to\ \eqref{eq: part 1 in regression ass} in Assumption\ \ref{ass:regression estimate consistency}.
For $(ii)$, we first derive 
\begin{align}\label{eq:hat epsilon - epsilon}
    \hat{\epsilon}_N^{(i)} - \epsilon^{(i)} &= \bigl(s'^{(i)} - \hat{f}_N(s^{(i)},a^{(i)})\bigr)-\bigl(s'^{(i)} - f^\star(s^{(i)},a^{(i)})\bigr)\nonumber\\
    & =f^\star(s^{(i)},a^{(i)}) - \hat{f}_N(s^{(i)},a^{(i)}).
\end{align}
Then we have
\small
\begin{align}
&\left|\frac{\gamma}{N} \sum_{i=1}^N \Bigl(V^\star\bigl(\Pi_{\mathcal S}(f^\star(s,a)+\hat\epsilon_N^{(i)})\bigr)
-V^\star\bigl(f^\star(s,a)+\epsilon^{(i)}\bigr)\Bigr)\right|\nonumber\\
&\overset{(a)}{\le} \frac{\gamma}{N} \sum_{i=1}^N L_{Q^\star}\Bigl\|\bigl(f^\star(s,a)+\hat\epsilon_N^{(i)}\bigr) - \bigl(f^\star(s,a)+\epsilon^{(i)}\bigr)\Bigr\|\nonumber\\
& \overset{(b)}{\le} \frac{\gamma}{N} \sum_{i=1}^N L_{Q^\star}\bigl\|f^\star(s^{(i)},a^{(i)}) - \hat{f}_N(s^{(i)},a^{(i)})\bigr\|\nonumber
\end{align}
\normalsize
where $(a)$ is due to Proposition\ \ref{prop: Q lip continuous} and the Lipschitz continuity of the orthogonal projection and $(b)$ is due to\ \eqref{eq:hat epsilon - epsilon}. Taking supremum over $(s,a)$ on both sides and driving $N$ to $+\infty$, $(ii)$ converges to 0 due to\ \eqref{eq: part 2 in regression ass} in Assumption\ \ref{ass:regression estimate consistency}. Finally, taking the supremum over $(s,a)$, $(iii)$ converges to $0$ in probability due to Assumption\ \ref{ass:LLN}. Overall, this completes the proof. 
\end{proof}
\begin{corollary}[Consistency of the value function $\hat{V}_N$]\label{coro:vn consistency}
Suppose Assumptions\ \ref{ass:cost}-\ref{ass:LLN} hold. Define the value function $ \hat{V}_N(s)=\min_{a\in\mathcal A} \hat{Q}_N(s,a)$.
Then, $\|\hat{V}_N - V^\star\|_{\infty} \xrightarrow{p} 0$ as $N \to \infty.$
\end{corollary}
\begin{proof}
From the definitions of $V^\star$ and $\hat{V}_N$, we  have
$\sup_{s\in\mathcal S}|\hat{V}_N(s)-V^\star(s)| =\sup_{s\in\mathcal S}\left|\min_{a\in\mathcal A}\hat{Q}_N(s,a)-\min_{a\in\mathcal A}Q^\star(s,a)\right| \le \sup_{s\in\mathcal S,a\in\mathcal A}\bigl|\hat{Q}_N(s,a)-Q^\star(s,a)\bigr|\xrightarrow{p} 0 \ \text{as } N\to\infty$.
\end{proof}

To develop the finite sample guarantee for $\hat{Q}_N$, we make the following uniform exponential bound assumption on the full-information Q-function $Q_N^\star$ and the finite-sample assumption on the regression estimate.
\begin{assumption}\label{ass:Q* - Q constant}
The full-information fixed point possess the following uniform exponential bound: for any constant $\eta$, there exist positive constants $A(\eta)$ and $\alpha(\eta)$ such that 
$$\mathbb P\left\{\sup_{(s,a) \in \mathcal{S}\times \mathcal{A}} |Q_N^\star(s,a) - Q^\star(s,a)| > \eta\right\} \le A(\eta) e^{-N \alpha(\eta)}.$$
\end{assumption} 

\begin{assumption}\label{ass:regression constant}
The regression estimate $\hat{f}_N$ possess the following finite sample properties: for any constant $\eta>0$ and $N\in\mathbb N$, there exist positive constants $A_f(\eta),\ \bar{A}_f(\eta),\ \alpha_f(N,\eta)$, and $\bar{\alpha}_f(N,\eta)$ with $\lim_{N \to\infty} \alpha_f(N,\eta) = \infty,\ \lim_{N \to\infty}\bar{\alpha}_f(N,\eta) = \infty$ such that
\footnotesize
\begin{align*}
& \mathbb P\left\{\sup_{(s,a) \in \mathcal S \times \mathcal A}\|f^\star(s,a) - \hat{f}_N(s,a)\| > \eta\right\} \le A_f(\eta)e^{-\alpha_f(N,\eta)},\\ 
& \mathbb P\left\{\frac{1}{N} \sum_{i=1}^N \|f^\star(s^{(i)},a^{(i)}) - \hat{f}_N(s^{(i)},a^{(i)})\| > \eta\right\} \le \bar{A}_f(\eta)e^{-\bar{\alpha}_f(N,\eta)}.
\end{align*}
\normalsize
\end{assumption}

Assumption\ \ref{ass:regression constant} can be satisfied by a wide range of parametric and nonparametric regression methods if we assume both $\mathcal S$ and $\mathcal A$ are compact. See more discussions in Appendix F in\ \cite{kannan2025data} and Assumption 6 in\ \cite{zhu2024residuals}. 

Next, we define the discrepancy in the $i$-th scenario between the residuals-based and full-information prediction as $\widetilde{\epsilon}_N^{(i)}(s,a):=(\hat f_N(s,a)+\hat\epsilon_N^{(i)}) - (f^\star(s,a) + \epsilon^{(i)})$. We show that the distance between $\hat{Q}_N$ and $Q_N^\star$ can be upper bounded by the norm of $\widetilde{\epsilon}_N^{(i)}(s,a)$ in the next lemma.
\begin{lemma}\label{lemma: hatQ_N - Q_N^*} Under Assumptions\ \ref{ass:cost} and\ \ref{ass:f star lip}, the following inequality holds:
\begin{align}
\left\|\hat{Q}_N - Q_N^\star\right\|_\infty \le \frac{\gamma L_{Q_N^\star}}{1-\gamma}   \sup_{(s,a) \in \mathcal S \times \mathcal A}\left(\frac{1}{N} \sum_{i=1}^N \|\widetilde{\epsilon}_N^{(i)}(s,a)\|\right).
\end{align}
\end{lemma}
\begin{proof}
Recall that $\hat{Q}_N,\ Q^\star_N$ are the fixed points of the residuals-based ($\hat{T}_N$) and full-information ($T^\star_N$) Bellman optimality operator, i.e., $\hat{Q}_N = \hat{T}_N(\hat{Q}_N)$, $Q^\star_N = T_N^\star(Q^\star_N)$. Following the same proof for\ \eqref{eq: Q and T} and denoting $V_N^\star(s)=\min_{a\in\mathcal A}Q_N^\star(s,a)$, we have
\allowdisplaybreaks
\footnotesize
\begin{align}
&\bigl\|\hat{Q}_N-Q_N^\star\bigr\|_\infty
\le \frac{1}{1-\gamma}
\bigl\|(\hat{T}_N-T_N^\star)Q_N^\star\bigr\|_\infty\nonumber\\
&\le \frac{\gamma}{1-\gamma}
\sup_{s,a}
\frac{1}{N}\sum_{i=1}^N
\Bigl|V_N^\star\bigl(\Pi_{\mathcal S}(\hat f_N(s,a)+\hat\epsilon_N^{(i)})\bigr)
-V_N^\star\bigl(f^\star(s,a)+\epsilon^{(i)}\bigr)
\Bigr| \nonumber\\
&\overset{(a)}{\le} \frac{\gamma L_{Q_N^\star}}{1-\gamma}
\sup_{s,a}
\frac{1}{N}\sum_{i=1}^N
\left\|
\Pi_{\mathcal S}\bigl(\hat{f}_N(s,a)+\hat{\epsilon}_N^{(i)}\bigr)
-
\bigl(f^\star(s,a)+\epsilon^{(i)}\bigr)
\right\|\nonumber\\
&\overset{(b)}{\le} \frac{\gamma L_{Q_N^\star}}{1-\gamma}
\sup_{s,a}
\frac{1}{N}\sum_{i=1}^N
\bigl\|\widetilde{\epsilon}_N^{(i)}(s,a)\bigr\|,
\label{eq:q-bound}
\end{align}
\normalsize
where $(a)$ holds due to Proposition\ \ref{prop:QNs_lip} and $(b)$ holds due to the Lipschitz continuity of the orthogonal projection.
\end{proof}
\begin{lemma}\label{lemma:finite sample lemma}
Suppose Assumptions\ \ref{ass:cost},\ \ref{ass:f star lip}, and\ \ref{ass:regression constant} hold. Then for any constant $\eta>0$ and $N \in \mathbb N$, there exist positive constants $\bar{A}(\eta)$ and $\bar{\alpha}(N,\eta)$ with $\lim_{N \rightarrow \infty} \bar{\alpha}(N,\eta) = \infty$ such that 
$$\mathbb P\left\{\sup_{(s,a) \in \mathcal S \times \mathcal A} |\hat{Q}_N(s,a) - Q_N^\star(s,a)| >\eta\right\} \le \bar{A}(\eta) e^{-\bar{\alpha}(N,\eta)}.$$
\end{lemma}
\begin{proof}
Lemma\ \ref{lemma: hatQ_N - Q_N^*} implies
\begin{align}
    &\mathbb P\left\{\sup_{(s,a) \in \mathcal S \times \mathcal A} |\hat{Q}_N(s,a) - Q_N^\star(s,a)| >\eta\right\} \nonumber\\
    \le & \mathbb P\left\{\frac{\gamma L_{Q_N^\star}}{1-\gamma} \sup_{(s,a) \in \mathcal S \times \mathcal A}\Bigl(\frac{1}{N} \sum_{i=1}^N \|\widetilde{\epsilon}_N^{(i)}(s,a)\|\Bigr) > \eta\right\} \nonumber \\
     \le &\mathbb P\left\{\sup_{(s,a) \in \mathcal S \times \mathcal A}\Bigl(\frac{1}{N} \sum_{i=1}^N \|\widetilde{\epsilon}_N^{(i)}(s,a)\|\Bigr) > \frac{\eta(1-\gamma)}{\gamma L_{Q^\star}}\right\} \nonumber
\end{align}
Then, with a change of variable, it suffices to show for any $\eta >0$ and $N\in\mathbb N$, there exist some constants $\bar{A}(\eta)>0$ and $\bar{\alpha}(N,\eta)>0$ such that
$$\mathbb P\left\{\sup_{(s,a) \in \mathcal S \times \mathcal A}\Bigl(\frac{1}{N} \sum_{i=1}^N \|\widetilde{\epsilon}_N^{(i)}(s,a)\|\Bigr) > \eta\right\} \le \bar{A}(\eta) e^{-\bar{\alpha}(N,\eta)}.$$
Recall that $\frac{1}{N}\sum_{i=1}^N \|\widetilde{\epsilon}_N^{(i)}(s,a)\| = \frac{1}{N}\sum_{i=1}^N \|(\hat{f}_N(s,a) + \hat{\epsilon}_N^{(i)}) - (f^\star(s,a) + \epsilon^{(i)})\| \le \|\hat{f}_N(s,a) - f^\star(s,a)\| + \frac{1}{N}\sum_{i=1}^N \|\hat{\epsilon}_N^{(i)}-\epsilon^{(i)}\| = \|\hat{f}_N(s,a) - f^\star(s,a)\| + \frac{1}{N}\sum_{i=1}^N \|f^\star(s^{(i)},a^{(i)}) - \hat{f}_N(s^{(i)},a^{(i)})\|$ where the last equality is due to\ \eqref{eq:hat epsilon - epsilon}.
Consequently, 
\begin{align}
&\mathbb P\left\{\sup_{(s,a) \in \mathcal S \times \mathcal A}\Bigl(\frac{1}{N} \sum_{i=1}^N \|\widetilde{\epsilon}_N^{(i)}(s,a)\|\Bigr) > \eta\right\}  \nonumber \\
 \le &\mathbb P\left\{\sup_{(s,a) \in \mathcal S \times \mathcal A} \|\hat{f}_N(s,a) - f^\star(s,a)\| > \eta/2 \right\} \nonumber\\
& + \mathbb P\left\{\frac{1}{N} \sum_{i=1}^N \|\hat{f}_N(s^{(i)},a^{(i)}) -f^\star(s^{(i)},a^{(i)}) \| > \eta/2 \right\}\nonumber\\
\overset{(a)}{\le}& A_f(\frac{\eta}{2})e^{-\alpha_f(N,\frac{\eta}{2})} + \bar{A}_f(\frac{\eta}{2})e^{-\bar{\alpha}_f(N,\frac{\eta}{2})} \nonumber\\
\le &\left(
A_f\!\left(\frac{\eta}{2}\right)
+\bar{A}_f\!\left(\frac{\eta}{2}\right)
\right)
e^{
-\min\!\left(
\alpha_f\!\left(N,\frac{\eta}{2}\right),
\bar{\alpha}_f\!\left(N,\frac{\eta}{2}\right)
\right)}\nonumber\\
=&\bar{A}(\eta)e^{-\bar{\alpha}(N,\eta)}\nonumber
\end{align}
where $(a)$ is due to Assumption\ \ref{ass:regression constant} and $\bar{\alpha}(N,\eta):=\min\{\alpha_f\!\left(N,\frac{\eta}{2}\right),
\bar{\alpha}_f\!\left(N,\frac{\eta}{2}\right)\}$ converges to 0 as $N\to\infty$ according to Assumption\ \ref{ass:regression constant}.
\end{proof}

\begin{theorem}[Finite sample guarantee of $\hat{Q}_N$]\label{thm:finite sample}
Suppose Assumptions\ \ref{ass:cost},\ \ref{ass:f star lip},\ \ref{ass:Q* - Q constant} and\ \ref{ass:regression constant} hold. Then for any constant $\eta>0$ and $N \in \mathbb N$, there exist positive constants $\Gamma(\eta)$ and $\gamma(N,\eta)$ with $\lim_{N\to\infty} \gamma(N,\eta)=\infty$ such that the fixed points $\hat{Q}_N$ and $Q^\star$ satisfy 
$$\mathbb P \left\{\|\hat{Q}_N-Q^\star\|_\infty > \eta \right\} \le \Gamma(\eta) \exp{(-\gamma(N,\eta))}.$$
\end{theorem}
\begin{proof}
For any $\eta>0$, we have
\begin{align}
&\mathbb P\!\left\{\|\hat Q_N-Q^\star\|_\infty>\eta\right\}\nonumber \\
\le&
\mathbb P\!\left\{\|\hat Q_N-Q_N^\star\|_\infty>\tfrac{\eta}{2}\right\}
+
\mathbb P\!\left\{\|Q_N^\star-Q^\star\|_\infty>\tfrac{\eta}{2}\right\}\nonumber \\ 
\overset{(a)}{\le}&
\bar A\!\left(\tfrac{\eta}{2}\right)\exp(-\bar\alpha\left(N,\tfrac{\eta}{2}\right))
+
A\!\left(\tfrac{\eta}{2}\right)\exp(-N\,\alpha\!\left(\tfrac{\eta}{2}\right)),\nonumber
\end{align}
where $(a)$ is due to Lemma\ \ref{lemma:finite sample lemma} and Assumption\ \ref{ass:Q* - Q constant}. Define $\Gamma(\eta)
:=\bar A\!\left(\tfrac{\eta}{2}\right)+A\!\left(\tfrac{\eta}{2}\right),\ \gamma(N,\eta)
:=\min\!\left\{\bar\alpha\!\left(N,\tfrac{\eta}{2}\right),\,N\alpha\!\left(\tfrac{\eta}{2}\right)\right\}.$
As a result, we have $\mathbb P\!\left\{\|\hat Q_N-Q^\star\|_\infty>\eta\right\}
\le
\Gamma(\eta)\exp(-\gamma(N,\eta))$,
which completes the proof.
\end{proof}
\begin{corollary}[Finite sample guarantee of $\hat{V}_N$]\label{coro:fsg for vn}
Suppose Assumptions\ \ref{ass:cost},\ \ref{ass:f star lip},\ \ref{ass:Q* - Q constant} and\ \ref{ass:regression constant} hold. Define the value function $\hat{V}_N(s)=\min_{a\in\mathcal A}\hat{Q}_N(s,a).$ Then for any constant $\eta>0$ and $N \in \mathbb N$, there exist positive constants $\Gamma(\eta)$ and $\gamma(N,\eta)$ such that
$\mathbb P\left\{\|\hat{V}_N-V^\star\|_\infty>\eta\right\}
\le
\Gamma(\eta)\exp\bigl(-\gamma(N,\eta)\bigr)$.
\end{corollary}

Theorem\ \ref{thm:finite sample} and Corollary\ \ref{coro:fsg for vn} illustrate that the probability of the optimality gap of the value function exceeding a prescribed tolerance decays exponentially with respect to $\gamma(N,\eta)$. 

\subsection{Residuals-based Offline Deep Q-Learning}
While Section \ref{sec:consistency} establishes the consistency and finite sample guarantees when we can evaluate the Q-function exactly, in this section, we introduce a practical residuals-based deep Q-learning (DQN) algorithm in Algorithm\ \ref{algo:residual based dqn} for continuous state and discrete action spaces. Compared to the traditional DQN algorithm discussed in\ \cite{mnih2013playing} and \cite{mnih2015human}, our approach first fits a regression model $\hat f_N$ using the offline dataset $D_N$ and then computes the corresponding empirical residuals. These residuals are used to construct a simulated transition kernel, where the next state $s_{t+1}$ is generated by adding a sampled residual to the model prediction $\hat f_N(s_t,a_t)$. Specifically, starting from any state-action pair $s_t, a_t$, we sample an empirical residual $\hat{\epsilon}_t$ from the residual set $\{\hat{\epsilon}_N^{(i)}\}_{i=1}^N$ and construct the next state as $s_{t+1}=\Pi_{\mathcal S}(\hat f_N(s,a)+\hat\epsilon_t)$. We then train our policy using the generated trajectories.
\begin{algorithm}[ht!]
\caption{Residuals-based Offline Deep Q-Learning.}
\begin{algorithmic}[1]
\label{algo:residual based dqn}
\scriptsize
\STATE \textbf{Input:} Dataset $D_N = \{(s^{(i)},a^{(i)},c^{(i)},s'^{(i)}\}_{i=1}^N$.
\STATE Initialize replay memory $\mathcal{D}$.
\STATE Fit a regression model $\hat{f}_N$ using dataset $D_N$ and calculate the empirical residuals $\{\hat{\epsilon}_N^{(i)}\}_{i=1}^N$, where $\hat{\epsilon}_N^{(i)} = s'^{(i)} - \hat{f}_N(s^{(i)},a^{(i)}),\ \forall i = 1,\ldots,N$.
\STATE Initialize action value function $Q$ with random weights $\theta$.
\STATE Initialize target action value function $\hat{Q}$ with random weights $\hat{\theta}$.
  \FOR {episode $k=1,2,\ldots$}
  \STATE Initialize state $s_1$.
    \FOR {time step $t=1,2,\ldots$}
    \STATE With probability $\epsilon$ select a random action $a_t\in\mathcal{A}$.
    \STATE Otherwise select $a_t=\arg\min_{a^{\prime}}Q(s_t,a^{\prime}; \theta)$.
    \STATE Sample an empirical residual $\hat{\epsilon}_t$ from the set $\{\hat \epsilon_N^{(i)}\}_{i=1}^N$ and update next state as $s_{t+1}= \Pi_{\mathcal S}(\hat f_N(s_t,a_t)+\hat\epsilon_t)$ and cost as $c_t=c(s_t,a_t)$.
    \STATE Store transition $(s_t, a_t, c_t, s_{t+1})$ in replay memory $\mathcal{D}$.
    \STATE Sample random minibatch of $M$ transitions  from $\mathcal{D}$: $\{(s_j, a_j, c_j, s_{j+1})\}_{j=1}^M$.
    \STATE Set $y_j = c_j + \gamma\min_{a^{\prime}}\hat{Q}(s_{j+1}, ,a^{\prime};\hat{\theta})$ for all $j=1,\ldots, M$.
      \STATE Calculate loss $\mathcal{L}=\frac{1}{M}\sum_{j=1}^M(y_j-Q(s_j, a_j; \theta))^2$.
      \STATE Perform a gradient descent step to minimize the loss $\mathcal{L}$ with respect to the network parameter $\theta$.
      \STATE Every $C$ steps reset $\hat{Q}=Q$.
  \ENDFOR
  \ENDFOR
\end{algorithmic}
\end{algorithm}

\section{Numerical Results}\label{sec:results}
We implement Algorithm\ \ref{algo:residual based dqn} in a stochastic variant of the CartPole environment, where we add a Gaussian noise $\epsilon\sim\mathcal N(0,0.25)$ to the observed cart position to enable stochastic transitions. We first collect transition data from this true environment and fit a neural network that takes the current state-action pair as input and predicts the next state, where we consider a single hidden layer with 64 neurons and ReLU activation. The Q-network is built as a fully connected neural network with two hidden layers of width 64 and ReLU activations. We train the Q-network within the simulated transition kernel using the \texttt{Adam} optimizer, with learning rate $10^{-4}$, and discount factor $\gamma=0.99$. We adopt an $\epsilon$-greedy action selection rule, where $\epsilon$ is initialized at 0.9 and decays exponentially to reach a minimum value of 0.01. For each run, the policy is trained for 1000 episodes in the residuals-based simulator and tested in the true environment. We report the average results over 5 independent runs.

\subsection{Comparison of Different Sample Sizes}
We evaluate the proposed residuals-based offline DQN algorithm under different number of trajectories sampled in the offline dataset, denoted by $n$. Specifically, we consider $n \in \{500,600,700,800,900,1000\}$. We report the average training and testing rewards from 5 independent runs in Fig.\ \ref{fig:compare sample sizes}.
\begin{figure}[ht!]
    \centering
    \begin{subfigure}[t]{0.45\textwidth}
        \centering
        \includegraphics[width=0.9\textwidth]{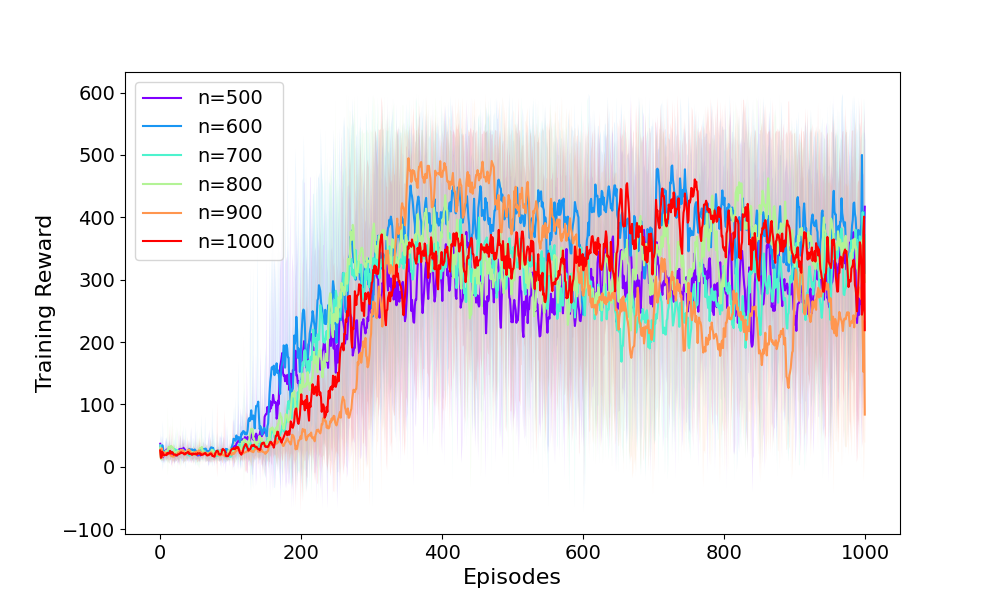}
        \caption{Training}
        \label{fig:different samples training}
    \end{subfigure}
    \hfill
    \begin{subfigure}[t]{0.45\textwidth}
        \centering
        \includegraphics[width=0.9\textwidth]{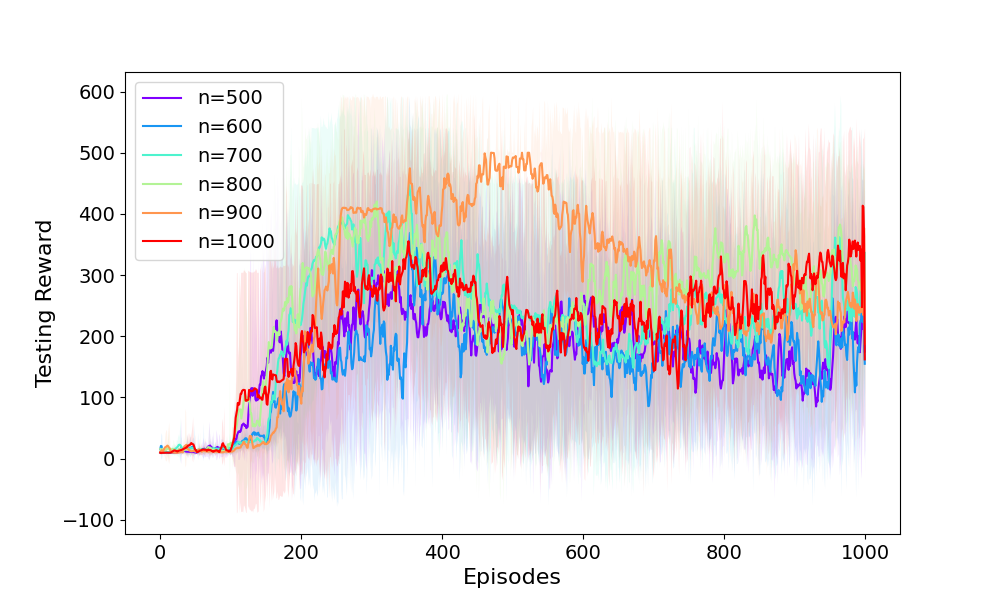}
        \caption{Testing}
        \label{fig:different samples evaluation}
    \end{subfigure}
    \caption{Comparison of different sample sizes}
    \label{fig:compare sample sizes}
\end{figure}
From Fig.\ \ref{fig:compare sample sizes}, both the training and testing rewards increase as $n$ increases, which indicates that increasing the amount of offline data improves the quality of the learned transition model and the effectiveness of the resulting policy in the testing environment. 
\subsection{Comparison of Different Models}
We also compare our model with a baseline approach in which the next state is predicted only using the regression $\hat{f}_N(s,a)$, without adding the empirical residuals. Fig.\ \ref{fig:model comparison} presents both the training and testing performances for our model (w/ residuals) and the benchmark (w/o residuals), where we take $n=500$ and $n=1000$.

\begin{figure}[ht!]
    \centering

    \begin{subfigure}[b]{0.45\textwidth}
        \centering
        \includegraphics[width=0.8\textwidth]{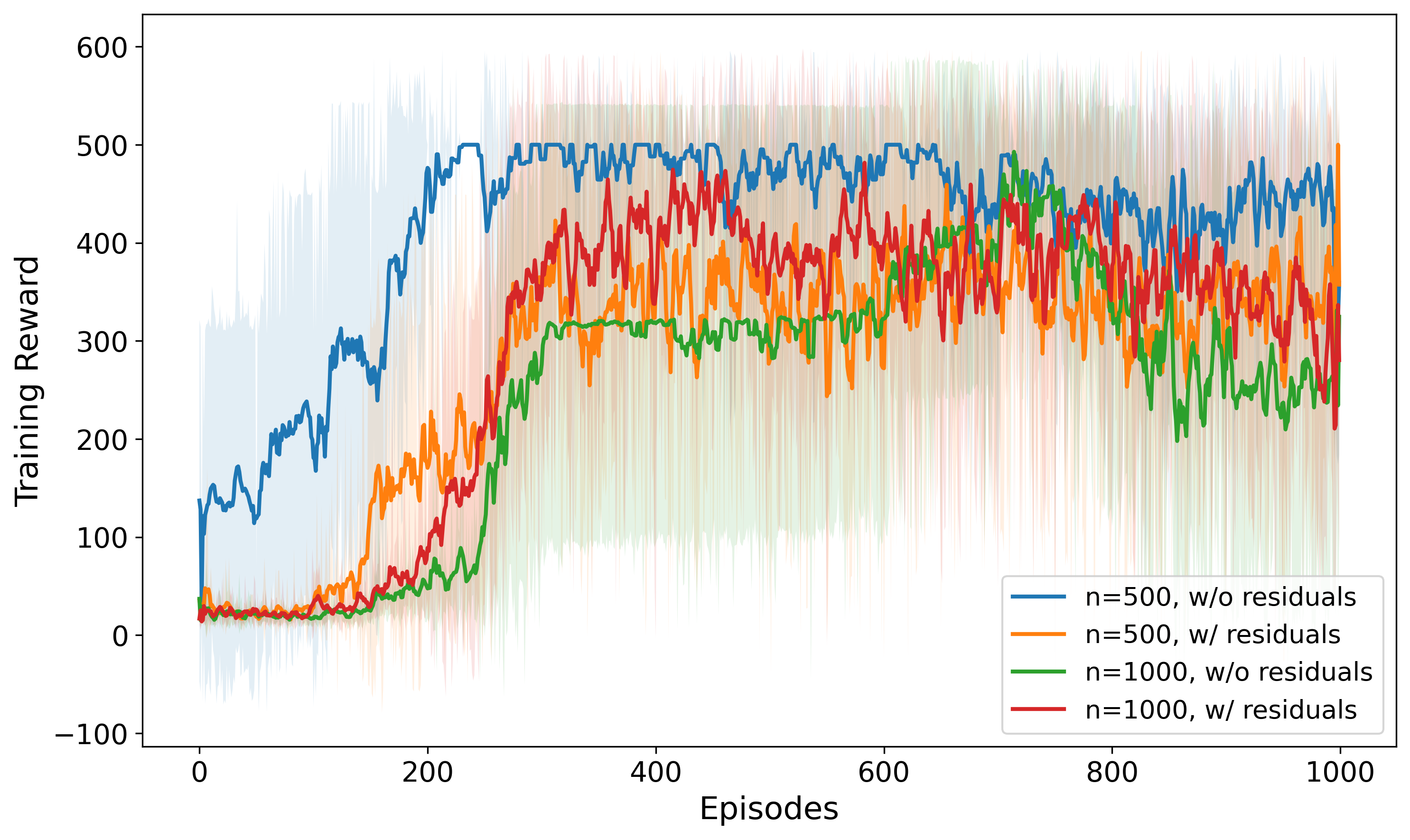}
        \caption{Training}
        \label{fig:sub1}
    \end{subfigure}
    \hfill
    \begin{subfigure}[b]{0.45\textwidth}
        \centering
        \includegraphics[width=0.8\textwidth]{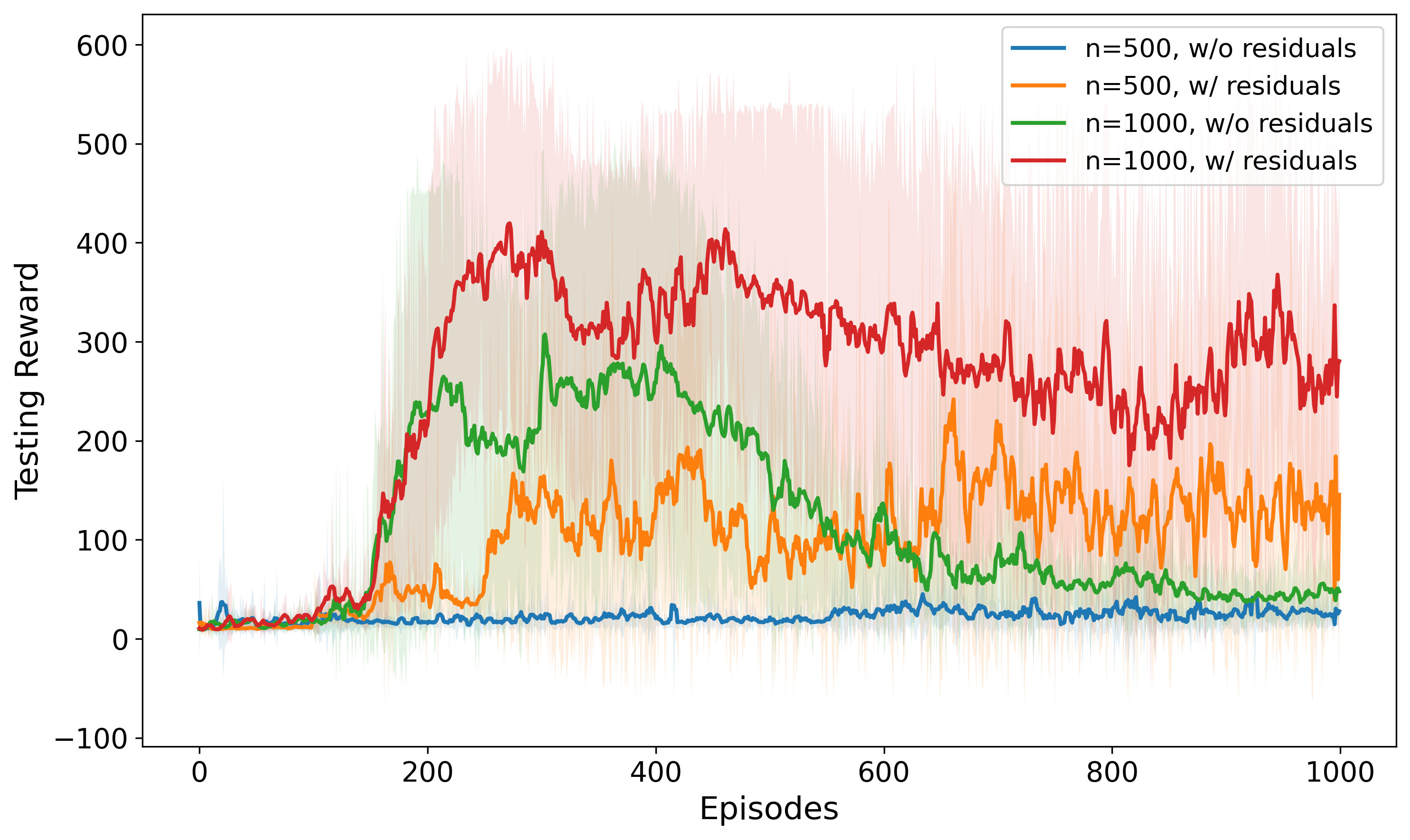}
        \caption{Testing}
        \label{fig:sub4}
    \end{subfigure}
    \caption{Comparison of Models}
    \label{fig:model comparison}
\end{figure}
Fig.\ \ref{fig:model comparison} shows that even if the baseline model achieves better training performance, it performs poorly in the true environment for both $n=500$ and $n=1000$. On the other hand, our residuals-based approach always outperforms by achieving a higher testing reward in the true environment. Specifically, when $n=1000$, the training and testing performances under the residuals-based approach are similar, which shows that it provides an accurate approximation of the underlying stochastic transition. This illustrates the benefit of accounting for estimation error in predicting the transition dynamics in offline RL.
\section{CONCLUSIONS}

In this paper, we introduced a residuals-based offline RL framework that directly links the estimation error in learning transition dynamics to downstream policy optimization. In particular, we first fit a regression model and constructed empirical residuals based on the offline dataset. Then, we defined a residuals-based Bellman optimality operator, where the transition kernel is approximated by the estimated regression model and empirical residuals. We proved that this residuals-based Bellman optimality operator is a contraction mapping and established consistency and finite-sample guarantees for its corresponding fixed point. For practical use, we also designed a residuals-based offline DQN algorithm and demonstrated its superior performance through numerical experiments on a stochastic CartPole environment. 
As a possible future direction, we will extend the proposed framework to an empirical residuals-based distributionally robust optimization setting for offline RL, which will account for distributional ambiguity and improve robustness for learning the uncertain transition dynamics.


\section{ACKNOWLEDGMENTS}

This work was supported in part by the National Science
Foundation under Grant \#2331782.


\bibliographystyle{IEEEtran}
\bibliography{bibl.bib, Xian_bib.bib}

\end{document}